\documentclass[11pt]{article}

\usepackage[utf8]{inputenc}
\usepackage[T1]{fontenc}
\usepackage{amsmath,amssymb,amsfonts,amsthm}
\usepackage{algorithmic}
\usepackage{algorithm}
\usepackage{graphicx}
\usepackage{booktabs}
\usepackage{hyperref}
\usepackage{xcolor}
\usepackage{natbib}
\usepackage{tikz}
\usetikzlibrary{shapes,arrows,positioning,fit,backgrounds,calc}
\usepackage{subcaption}
\usepackage{tcolorbox}
\usepackage[margin=1in]{geometry}
\usepackage{enumitem}

\newtheorem{theorem}{Theorem}
\newtheorem{corollary}{Corollary}
\newtheorem{proposition}{Proposition}

\newcommand{\method}{\textsc{SleepGate}}
\newcommand{\R}{\mathbb{R}}
\newcommand{\E}{\mathbb{E}}

\newcommand{\cC}{\mathcal{C}}
\newcommand{\cL}{\mathcal{L}}
\newcommand{\cS}{\mathcal{S}}
\newcommand{\cM}{\mathcal{M}}

\title{%
  \textbf{Learning to Forget: Sleep-Inspired Memory Consolidation for Resolving Proactive Interference in Large Language Models}
}

\author{%
  \textbf{Ying Xie} \\
  Kennesaw State University \\
  \texttt{yxie2@kennesaw.edu}
}

\date{}

\begin{document}
\maketitle

\begin{abstract}
Large language models (LLMs) suffer from \emph{proactive interference} (PI):
previously processed but now-outdated information in the context window
disrupts retrieval of current, relevant values. Recent work demonstrates that
this interference degrades retrieval accuracy log-linearly toward chance as
stale associations accumulate---a working-memory bottleneck that persists
regardless of context length and resists prompt-engineering mitigations.
Biological brains face an analogous challenge and resolve it through
sleep-dependent memory consolidation: an active, multi-stage process of
synaptic downscaling, selective replay, and targeted forgetting.

We propose \method{}, a biologically inspired framework that augments
transformer-based LLMs with a learned \emph{sleep cycle} operating over the
key-value (KV) cache. \method{} introduces three coordinated mechanisms:
(1)~a \emph{conflict-aware temporal tagger} that detects when new entries
supersede old ones; (2)~a lightweight \emph{forgetting gate} network trained
to selectively evict or compress stale cache entries; and (3)~a
\emph{consolidation module} that merges related surviving entries into compact
summary representations. These components are activated periodically during
inference in ``sleep micro-cycles,'' governed by an adaptive trigger based on
attention entropy. We formalize a dual-phase training objective that jointly
optimizes standard language modeling during the wake phase and
post-consolidation retrieval accuracy during the sleep phase, with explicit
pressure toward cache compression. We provide theoretical analysis showing
that \method{} can reduce the effective interference horizon from $O(n)$ to
$O(\log n)$ under mild assumptions, present a concrete algorithmic
specification, and provide preliminary experimental validation on a controlled
PI-LLM benchmark. In proof-of-concept experiments with a small-scale transformer
(4 layers, 793K parameters), \method{} achieves 99.5\% retrieval accuracy at
PI depth~5 and 97.0\% at depth~10, while all five baselines---full KV cache,
sliding window, H\textsubscript{2}O, StreamingLLM, and a decay-only
ablation---remain below 18\% across all depths. Our framework offers an
architecture-level solution to a limitation that prompt
engineering cannot address.
\end{abstract}

\vspace{0.5em}

\section{Introduction}
\label{sec:intro}

The dominant paradigm for extending the memory capacity of large language
models has been to increase context window size substantially in successive
generations of models~\citep{radford2019language, anthropic2024claude,
openai2023gpt4turbo, google2024gemini}.
The implicit assumption is that if a model can \emph{attend} to more tokens,
it can \emph{use} more information. However, accumulating evidence
suggests that this assumption is flawed. Models exhibit systematic retrieval
failures even when target information lies well within the context
window~\citep{liu2024lost, hsieh2024ruler}, and performance degrades not just
with distance but with the \emph{amount of competing information}.

\citet{wang2025unable} recently provided a clear demonstration of this
phenomenon through the lens of \emph{proactive interference} (PI), a
well-established construct in cognitive psychology~\citep{underwood1957interference,
wickens1970encoding}. In their PI-LLM paradigm, models receive a stream of
semantically related key-value pairs where later entries overwrite earlier
ones, and are then queried on the final (most recent) value. Despite the
target answer being positioned immediately before the query, retrieval
accuracy declines log-linearly toward zero as the number of prior
(now-superseded) associations increases. The errors are not
random---models systematically retrieve \emph{overwritten} values,
demonstrating that stale information actively competes with and suppresses
current information. Prompt-engineering interventions (e.g., instructing the
model to ignore prior values) provide only marginal relief.

This finding reveals a \emph{working memory bottleneck} that is fundamentally
distinct from context length limitations. The model can see the relevant tokens---the difficulty is
\emph{suppressing} the irrelevant ones. In the standard transformer attention
mechanism, every entry in the key-value cache participates in every attention
computation with no mechanism for selective inhibition. As stale entries
accumulate, they collectively drown out the signal from current, relevant
entries through sheer numerical mass.

\subsection{The Biological Precedent: Sleep as Active Forgetting}
\label{sec:intro_bio}

Biological neural systems face the same challenge. During waking
hours, the brain continuously encodes new associations, many of which
conflict with or supersede earlier ones. Left unchecked, this accumulation
would produce catastrophic interference~\citep{mccloskey1989catastrophic,
french1999catastrophic}. The brain's solution is \emph{sleep}---not as
passive downtime, but as an active computational process dedicated to memory
management.

The synaptic homeostasis hypothesis (SHY)~\citep{tononi2006sleep,
tononi2014sleep} proposes that waking experience produces a net increase in
synaptic strength, and that sleep restores homeostasis through global
synaptic downscaling: all synapses are proportionally weakened, preserving
relative strength differences while reducing absolute levels. This prevents
saturation and improves signal-to-noise ratios.

Complementing this global process, the brain engages in \emph{selective
memory replay} during sleep, particularly during slow-wave sleep
(SWS)~\citep{diekelmann2010memory, rasch2013sleep}. The hippocampus replays
recently encoded experiences to the neocortex, promoting consolidation of
important memories while allowing non-replayed traces to decay. Sleep has been shown to reduce vulnerability to associative interference and
protect recently formed memories~\citep{ellenbogen2006human}, consistent with
its broader role in memory consolidation and plasticity
\citep{abel2013sleep}.

Active forgetting mechanisms further supplement these processes.
Neuromodulatory and circuit-level processes can contribute to the weakening of
memory traces and adaptive forgetting~\citep{berry2014sleep, davis2017biology},
while sleep-specific neural oscillations such as spindles and sharp-wave
ripples help coordinate memory consolidation-related information transfer
during sleep~\citep{staresina2015hierarchical}.

\subsection{Our Proposal}
\label{sec:intro_proposal}

We propose to endow LLMs with an analogous capability: a learned, periodic
\emph{sleep cycle} that actively manages the key-value cache.
Our framework, \method{}, operates at the architectural level---modifying how
the model maintains its working memory instead of relying on
prompting or post-hoc filtering.

The contributions of this paper are:
\begin{enumerate}[leftmargin=*]
  \item \textbf{A biologically grounded framework} for active KV cache
    management that maps the three core mechanisms of sleep-dependent memory
    consolidation (synaptic downscaling, selective replay, active forgetting)
    onto concrete computational modules (\S\ref{sec:method}).
  \item \textbf{A formal dual-phase training objective} that jointly optimizes
    language modeling performance and post-consolidation retrieval accuracy,
    with explicit compression pressure (\S\ref{sec:training}).
  \item \textbf{Theoretical analysis} showing that the proposed mechanism can
    reduce the effective PI horizon from linear to logarithmic in the number
    of superseding updates (\S\ref{sec:theory}).
  \item \textbf{Preliminary experimental validation} on a controlled PI-LLM
    benchmark demonstrating that \method{} outperforms all
    baselines by a wide margin across seven PI depths, with analysis of failure modes at
    extreme interference levels
    (\S\ref{sec:experiments}, \S\ref{sec:results}).
\end{enumerate}

\section{Related Work}
\label{sec:related}

\paragraph{Context Window Limitations and the Lost-in-the-Middle Effect.}
\citet{liu2024lost} demonstrated that LLMs struggle to retrieve information
from the middle of long contexts, even when the information is present.
Subsequent work on long-context benchmarks~\citep{hsieh2024ruler,
li2024longbench} has shown that scaling context length does not
proportionally improve downstream task performance. These findings point to the need
for mechanisms beyond simple context extension.

\paragraph{Proactive Interference in Human Cognition.}
PI is one of the most robust phenomena in memory
research~\citep{underwood1957interference}. The buildup-release paradigm
of~\citet{wickens1970encoding} showed that PI accumulates with similar
material and is released by categorical shifts. \citet{kane2000role} linked PI susceptibility to working memory capacity, and \citet{ellenbogen2006human} demonstrated that sleep reduces susceptibility to
associative interference and protects recently encoded memories.

\paragraph{Proactive Interference in LLMs.}
\citet{wang2025unable} introduced the PI-LLM paradigm and demonstrated
log-linear accuracy degradation under PI across multiple model families.
Their work established that this is an architectural bottleneck, not a
training data or prompting issue. Our work builds directly on this finding by
proposing an architectural solution.

\paragraph{KV Cache Optimization.}
Several lines of work address KV cache efficiency. Local or sliding-window
attention~\citep{beltagy2020longformer} limits attention to a fixed window but
discards information indiscriminately. Other sparse attention mechanisms use
structured patterns to reduce attention cost~\citep{child2019generating}.
H\textsubscript{2}O~\citep{zhang2023h2o} retains ``heavy hitter'' tokens
based on cumulative attention scores. StreamingLLM~\citep{xiao2024efficient}
maintains attention sinks plus a sliding window. \citet{ge2024model} propose
model-driven cache compression. These approaches optimize for
\emph{efficiency} (reducing memory footprint) rather than \emph{interference
resolution} (selectively removing stale information that degrades accuracy).
\method{} addresses the latter while naturally achieving the former as a
side effect.

\paragraph{Memory-Augmented Architectures.}
External memory systems~\citep{graves2014neural, sukhbaatar2015end,
wu2022memorizing} and retrieval-augmented generation~\citep{lewis2020retrieval,
borgeaud2022improving} extend model memory beyond the context window.
However, these systems typically \emph{add} memory capacity without
addressing the interference problem within the existing context.
\citet{munkhdalai2024leave} propose leaving context behind via a
compressive memory, which is complementary to our approach.

\paragraph{Biological Inspiration in Machine Learning.}
Sleep-inspired computation has been explored in continual learning \citep{tadros2022sleep, gonzalez2020can}, where "sleep" phases involving generative replay help consolidate knowledge across tasks. \citet{kumaran2016learning} revisit complementary learning systems theory, showing how hippocampal--neocortical interactions can mitigate interference. Our work differs in targeting \emph{in-context} interference during inference rather than \emph{cross-task} interference during training.

\paragraph{Forgetting Mechanisms.}
Sparse attention~\citep{tay2022efficient}, gated retention~\citep{sun2023retentive},
and state-space models~\citep{gu2023mamba} implicitly implement forms of
forgetting through their inductive biases. However, these are fixed
architectural choices rather than learned, content-dependent forgetting
policies. The forgetting gate in LSTM~\citep{hochreiter1997long} is the
closest architectural precedent to our proposal, but operates at the
hidden-state level within a single timestep, not over an extended KV
cache.

\section{The \method{} Framework}
\label{sec:method}

We describe \method{} in detail. The framework augments a standard
transformer-based LLM with three modules that operate over the KV cache,
orchestrated by an adaptive scheduling mechanism. Figure~\ref{fig:architecture}
provides an overview.

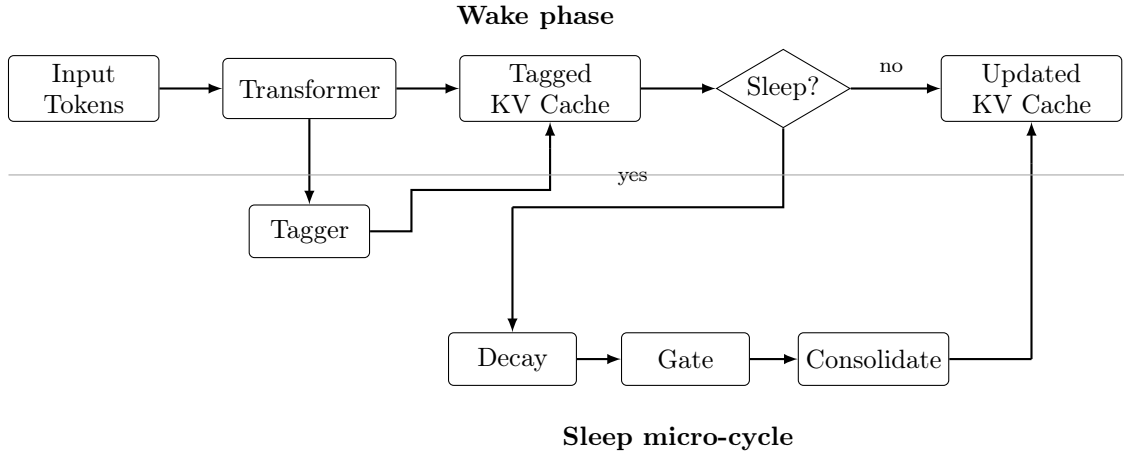
\begin{figure}[t]
\centering
\begin{tikzpicture}[
    font=\small,
    box/.style={
        rectangle,
        draw,
        rounded corners=2pt,
        minimum height=8mm,
        minimum width=18mm,
        align=center,
        inner sep=3pt
    },
    smallbox/.style={
        rectangle,
        draw,
        rounded corners=2pt,
        minimum height=7mm,
        minimum width=15mm,
        align=center,
        inner sep=2pt
    },
    decision/.style={
        diamond,
        draw,
        aspect=1.7,
        align=center,
        inner sep=1.5pt
    },
    phaselabel/.style={font=\bfseries\small},
    yn/.style={font=\scriptsize}
]

\node[box, minimum width=20mm] (input) at (0,0) {Input\\Tokens};
\node[box, minimum width=23mm] (transformer) at (3.0,0) {Transformer};
\node[box, minimum width=24mm] (cache) at (6.2,0) {Tagged\\KV Cache};
\node[decision, minimum width=15mm, minimum height=9mm] (trigger) at (9.3,0) {Sleep?};
\node[box, minimum width=24mm] (output) at (12.6,0) {Updated\\KV Cache};

\node[smallbox, minimum width=16mm] (tagger) at (3.0,-1.9) {Tagger};
\node[smallbox, minimum width=17mm] (decay) at (5.7,-3.6) {Decay};
\node[smallbox, minimum width=17mm] (gate) at (8.0,-3.6) {Gate};
\node[smallbox, minimum width=20mm] (cons) at (10.5,-3.6) {Consolidate};

\draw[thick,->,>=latex] (input.east) -- (transformer.west);
\draw[thick,->,>=latex] (transformer.east) -- (cache.west);
\draw[thick,->,>=latex] (cache.east) -- (trigger.west);
\draw[thick,->,>=latex] (trigger.east) -- (output.west);
\node[yn] at ($(trigger.east)!0.45!(output.west)+(0,0.28)$) {no};

\draw[thick,->,>=latex] (transformer.south) -- (tagger.north);

\coordinate (tagA) at ($(tagger.east)+(0.55,0)$);
\coordinate (tagB) at ($(tagA)+(0,0.55)$);
\coordinate (tagC) at ($(cache.south |- tagB)$);
\coordinate (tagD) at ($(cache.south)+(0,-0.35)$);

\draw[thick] (tagger.east) -- (tagA) -- (tagB) -- (tagC) -- (tagD);
\draw[thick,->,>=latex] (tagD) -- (cache.south);

\coordinate (sleepturn) at ($(trigger.south)+(0,-1.05)+(-3.6,0)$);
\draw[thick] (trigger.south) -- ++(0,-1.05) -- (sleepturn);
\draw[thick,->,>=latex] (sleepturn) -- (decay.north);
\node[yn] at (7.3,-1.2) {yes};

\draw[thick,->,>=latex] (decay.east) -- (gate.west);
\draw[thick,->,>=latex] (gate.east) -- (cons.west);

\coordinate (stub) at ($(output.south)+(0,-0.55)$);
\draw[thick] (cons.east) -- (12.6,-3.6);
\draw[thick] (12.6,-3.6) -- (stub);
\draw[thick,->,>=latex] (stub) -- (output.south);

\node[phaselabel] at (6.0,0.95) {Wake phase};
\node[phaselabel] at (7.9,-4.65) {Sleep micro-cycle};

\draw[black!35] (-1.0,-1.15) -- (13.9,-1.15);

\end{tikzpicture}
\caption{Simplified overview of \method{}. During the wake phase, the model
processes input tokens while maintaining a tagged KV cache. A sleep trigger
decides whether to continue normal inference or enter a sleep micro-cycle. In
the soft biasing variant used in our experiments, the cache passes through
decay and the forgetting gate, and the resulting retention scores are converted
to additive attention biases (Eq.~\ref{eq:soft_bias}) for a second forward
pass. In the hard eviction variant, entries additionally pass through
consolidation before eviction.}
\label{fig:architecture}
\end{figure}

\subsection{Preliminaries and Notation}
\label{sec:preliminaries}

Consider a transformer model processing a sequence of tokens
$\mathbf{x} = (x_1, \ldots, x_T)$. At each layer $\ell$ and head $h$, the
model maintains a KV cache $\cC^{\ell,h} = \{(\mathbf{k}_i, \mathbf{v}_i)\}_{i=1}^{t}$
where $\mathbf{k}_i, \mathbf{v}_i \in \R^d$ are the key and value vectors
for position $i$, and $t \leq T$ is the current position. The attention
output for query $\mathbf{q}_t$ is:
\begin{equation}
  \text{Attn}(\mathbf{q}_t, \cC) = \sum_{i=1}^{t}
  \frac{\exp(\mathbf{q}_t^\top \mathbf{k}_i / \sqrt{d})}
       {\sum_{j=1}^{t} \exp(\mathbf{q}_t^\top \mathbf{k}_j / \sqrt{d})}
  \mathbf{v}_i
  \label{eq:attention}
\end{equation}

In the PI setting of~\citet{wang2025unable}, the sequence contains $n$
updates to the same semantic key, each with a different value:
$(k, v_1), (k, v_2), \ldots, (k, v_n)$, and the model is queried for $v_n$.
All prior entries $(k, v_1), \ldots, (k, v_{n-1})$ are \emph{stale} and
constitute interference.

\subsection{Module 1: Conflict-Aware Temporal Tagger}
\label{sec:tagger}

The first module augments each KV cache entry with metadata that enables
downstream conflict detection and staleness identification. We extend the
standard cache to an \emph{augmented cache}:
\begin{equation}
  \cC^+ = \left\{ \left(\mathbf{k}_i, \mathbf{v}_i, \tau_i,
  \mathbf{s}_i, \sigma_i, a_i \right) \right\}_{i=1}^{t}
  \label{eq:augmented_cache}
\end{equation}
where $\tau_i \in \mathbb{N}$ is the position timestamp, $\mathbf{s}_i \in \R^{d_s}$
is a \emph{semantic signature} vector, $\sigma_i \in \{0, 1\}$ is a
binary \emph{superseded flag}, and $a_i \in \mathbb{R}_{\ge 0}$ is the
cumulative attention received by entry $i$.

\paragraph{Semantic Signatures.}
The semantic signature $\mathbf{s}_i$ captures what ``slot'' or ``entity''
the entry refers to, abstracting away the specific value. We compute it via a
lightweight projection of the key vector combined with local context:
\begin{equation}
  \mathbf{s}_i = \text{LayerNorm}\!\left(
    W_s \left[\mathbf{k}_i \,\|\, 
    \text{LocalPool}\big(\{\mathbf{k}_j\}_{j=\max(1,i-w)}^{\min(t,i+w)}\big)\right]
  \right)
  \label{eq:semantic_sig}
\end{equation}
where $W_s \in \R^{d_s \times (d + d)}$ is a learned projection,
$\text{LocalPool}$ averages keys in a window of size $2w+1$, and $\|\,$
denotes concatenation.

\paragraph{Conflict Detection.}
An entry $i$ is marked as potentially superseded ($\sigma_i = 1$) when a later entry has sufficiently high semantic similarity:
\begin{equation}
  \sigma_i = \mathbf{1}\!\left[
    \exists\, j > i : \cos(\mathbf{s}_i, \mathbf{s}_j) > \delta
  \right]
  \label{eq:conflict}
\end{equation}
where $\delta$ is a learned or tuned threshold. In practice, we maintain a
running set of active semantic signatures and check incoming entries against
it, achieving $O(1)$ amortized cost per token using locality-sensitive
hashing~\citep{indyk1998approximate}.

\subsection{Module 2: Forgetting Gate}
\label{sec:gate}

The forgetting gate $G_\theta: \R^{d_g} \to [0, 1]$ is a small neural
network that assigns a \emph{retention score} to each cache entry,
determining whether it should be kept, compressed, or evicted. This is the
core ``active forgetting'' mechanism, analogous to the selective synaptic
downscaling and dopaminergic erasure processes observed during sleep.

For each entry $i$ in the augmented cache, we compute an input feature vector:
\begin{equation}
  \mathbf{f}_i = \left[
    \mathbf{k}_i \,\|\,
    \mathbf{v}_i \,\|\,
    \text{PE}(\tau_i, t) \,\|\,
    \mathbf{s}_i \,\|\,
    \sigma_i \,\|\,
    a_i \,\|\,
    \bar{\mathbf{c}}
  \right]
  \label{eq:gate_input}
\end{equation}
where $\text{PE}(\tau_i, t)$ is a relative positional encoding capturing the
age of the entry, $a_i$ is the cumulative attention received by this entry
(analogous to the ``heavy hitter'' score of~\citet{zhang2023h2o}), and
$\bar{\mathbf{c}}$ is a global context summary (mean-pooled over recent
entries).

The retention score and resulting action are:
\begin{equation}
  r_i = G_\theta(\mathbf{f}_i), \qquad
  \text{action}_i = \begin{cases}
    \textsc{Keep}     & \text{if } r_i \geq \alpha_k \\
    \textsc{Compress} & \text{if } \alpha_e \leq r_i < \alpha_k \\
    \textsc{Evict}    & \text{if } r_i < \alpha_e
  \end{cases}
  \label{eq:gate_action}
\end{equation}
where $\alpha_k > \alpha_e$ are learned or tuned thresholds.

\paragraph{Architecture of $G_\theta$.}
To minimize overhead, $G_\theta$ is implemented as a 2-layer MLP with GeLU
activations and a sigmoid output:
\begin{equation}
  r_i = G_\theta(\mathbf{f}_i) =
  \text{sigmoid}\!\left(
    w_r^\top \text{GeLU}(W_1 \mathbf{f}_i + b_1) + b_r
  \right)
  \label{eq:gate_arch}
\end{equation}
with $W_1 \in \R^{d_h \times d_g}$, $w_r \in \R^{d_h}$, and hidden
dimension $d_h = 128$. The total parameter count is negligible
relative to the base model (typically $< 0.01\%$).

\paragraph{Differentiable Training.}
For differentiable training of the discrete keep/compress/evict decision, we
add a 3-way action head on top of the shared hidden representation and apply a
Gumbel-softmax relaxation:
\begin{equation}
\mathbf{z}_i = W_a \,\text{GeLU}(W_1 \mathbf{f}_i + b_1) + b_a \in \mathbb{R}^3
\label{eq:action_head}
\end{equation}
\begin{equation}
\hat{\mathbf{y}}_i =
\text{softmax}\!\left(
  \frac{\mathbf{z}_i + \mathbf{g}_i}{\tau_{\text{temp}}}
\right)
\label{eq:gumbel_softmax}
\end{equation}
where $\mathbf{z}_i$ denotes the logits for the keep, compress, and evict
actions, $\mathbf{g}_i \in \mathbb{R}^3$ is a sampled Gumbel noise vector whose
components are independently drawn from a Gumbel distribution, and
$\tau_{\text{temp}}$ is the temperature. At inference time, we apply the hard
decision rule in Eq.~\ref{eq:gate_action}.

When using soft attention biasing (below), only the scalar retention score
$r_i$ is needed rather than a 3-way action distribution. In this case, the
Gumbel-softmax relaxation simplifies to a \emph{Gumbel-sigmoid} (binary
concrete) relaxation applied directly to the gate output:
$r_i = \sigma\!\left((\ell_i + g_i) / \tau_{\text{temp}}\right)$,
where $\ell_i$ is the scalar logit from Eq.~\ref{eq:gate_arch} (before
the sigmoid), $g_i \sim \text{Gumbel}(0,1)$, and
$\sigma(\cdot)$ denotes the sigmoid function.

\paragraph{Soft Attention Biasing.}
Rather than applying hard eviction decisions at inference time, we introduce a
\emph{soft attention biasing} mechanism that uses the retention scores as
continuous modifiers of the attention computation. For each cache entry~$i$,
we compute an additive pre-softmax bias:
\begin{equation}
  b_i = \beta \cdot \log\!\left(\max(r_i, \varepsilon)\right)
  \label{eq:soft_bias}
\end{equation}
where $\beta > 0$ is a scale hyperparameter and $\varepsilon$ is a small
constant for numerical stability. The modified attention becomes:
\begin{equation}
  \text{Attn}_{\text{sleep}}(\mathbf{q}_t, \cC) = \sum_{i=1}^{t}
  \frac{\exp\!\left(\mathbf{q}_t^\top \mathbf{k}_i / \sqrt{d} + b_i\right)}
       {\sum_{j=1}^{t} \exp\!\left(\mathbf{q}_t^\top \mathbf{k}_j / \sqrt{d} + b_j\right)}
  \mathbf{v}_i
  \label{eq:biased_attention}
\end{equation}
Since $r_i \approx 0$ for stale entries yields $b_i \ll 0$, their attention
weights are exponentially suppressed without being removed entirely.
Because the bias is continuous, this approach requires no Gumbel-softmax
relaxation during joint training and no threshold calibration stage. It also
degrades gracefully: entries are downweighted, not deleted, so the model can
recover from gate errors. In practice, setting $\beta = 5$
provides sufficient suppression: an entry with $r_i = 0.01$ receives a bias of
$b_i \approx -23$, effectively zeroing its attention weight.

\subsection{Module 3: Consolidation Module}
\label{sec:consolidation}

Entries assigned the \textsc{Compress} action are not simply discarded but
are \emph{consolidated} into compact summary representations, analogous to
hippocampal replay transferring episodic memories into semantic knowledge.

We group entries marked for compression into clusters based on their
semantic signatures using a simple greedy algorithm: entry $i$ joins the
cluster of the most similar active entry, or starts a new cluster if no
similarity exceeds $\delta/2$.

For each cluster $\cS_m = \{i_1, \ldots, i_{|\cS_m|}\}$, we produce a
consolidated key-value pair:
\begin{equation}
  \mathbf{k}_m^* =
  \frac{\sum_{i \in \cS_m} r_i \mathbf{k}_i}
       {\sum_{i \in \cS_m} r_i + \varepsilon},
  \qquad
  \mathbf{v}_m^* = \sum_{i \in \cS_m} \alpha_i^{(m)} W_V' \mathbf{v}_i
  \label{eq:consolidation}
\end{equation}
where $\varepsilon > 0$ is a small constant for numerical stability.
To preserve the most useful information within each cluster, we compute
recency-biased attention weights using a learned query vector
$\mathbf{q}_{\text{latest}}$:
\begin{equation}
  \alpha_i^{(m)} =
  \frac{
    \exp\!\left(
      \mathbf{q}_{\text{latest}}^\top W_K' \mathbf{k}_i / \sqrt{d}
      \;+\; \eta \cdot \tilde{\tau}_i
    \right)
  }{
    \sum_{j \in \cS_m}
    \exp\!\left(
      \mathbf{q}_{\text{latest}}^\top W_K' \mathbf{k}_j / \sqrt{d}
      \;+\; \eta \cdot \tilde{\tau}_j
    \right)
  }
  \label{eq:consol_attn}
\end{equation}
where $W_K'$ and $W_V'$ are learned projection matrices,
$\tilde{\tau}_i = \tau_i / \max_j \tau_j$ is the normalized timestamp, and
$\eta > 0$ is a recency weight (set to $\eta = 2$ in our implementation).
This operation compresses $|\cS_m|$ entries into a single entry, yielding a
compression ratio of $|\cS_m|{:}1$. Because $\mathbf{q}_{\text{latest}}$ is
biased toward recent entries, the consolidated representation tends to
preserve the most recent value within the cluster, which mitigates
proactive interference.

\subsection{Sleep Trigger: Adaptive Scheduling}
\label{sec:trigger}

The sleep micro-cycle must be triggered at the right frequency: too often
incurs unnecessary overhead; too rarely allows interference to accumulate.
We propose an adaptive trigger based on two complementary signals:

\paragraph{Signal 1: Attention Entropy.}
When PI accumulates, attention distributions become more uniform (the model
``doesn't know where to look''). We monitor the average attention entropy
across heads:
\begin{equation}
  H_t = -\frac{1}{|\mathcal{H}|} \sum_{h \in \mathcal{H}}
  \sum_{i \in \cC_t^+}
  \alpha_{t,i}^{(h)} \log \alpha_{t,i}^{(h)}
  \label{eq:entropy}
\end{equation}
where $\alpha_{t,i}^{(h)}$ is the attention weight from position $t$ to $i$
at head $h$. A sleep cycle is triggered when $H_t$ exceeds a running
threshold $\bar{H} + \kappa \cdot \text{std}(H)$.

\paragraph{Signal 2: Conflict Density.}
We also monitor the fraction of current cache entries marked as superseded:
\begin{equation}
  \rho_t = \frac{1}{|\cC_t^+|} \sum_{i \in \cC_t^+} \sigma_i
  \label{eq:conflict_density}
\end{equation}
A sleep cycle is triggered when $\rho_t > \rho_{\max}$.

The overall trigger condition is:
\begin{equation}
  \text{trigger}(t) = \left(H_t > \bar{H} + \kappa \cdot \text{std}(H)\right)
  \;\lor\; \left(\rho_t > \rho_{\max}\right)
  \;\lor\; \left(t \bmod N_{\max} = 0\right)
  \label{eq:trigger}
\end{equation}
where the last term provides a fallback periodic trigger every $N_{\max}$
tokens.

\subsection{The Complete Sleep Micro-Cycle}
\label{sec:cycle}

Algorithm~\ref{alg:sleep} specifies the complete procedure.

\begin{algorithm}[t]
\caption{Sleep Micro-Cycle}
\label{alg:sleep}
\begin{algorithmic}[1]
\REQUIRE Augmented KV cache $\cC^+$, forgetting gate $G_\theta$,
         thresholds $\alpha_k, \alpha_e$, decay rate $\lambda$, bias scale $\beta$
\ENSURE Updated cache $\cC'^+$ (hard variant) or attention bias $\mathbf{b}$ (soft variant)

\STATE \textbf{// Phase 1: Key Decay (Synaptic Downscaling)}
\FOR{each entry $(\mathbf{k}_i, \mathbf{v}_i, \tau_i, \mathbf{s}_i, \sigma_i, a_i) \in \cC^+$}
  \STATE $\text{age}_i \leftarrow t_{\text{current}} - \tau_i$
  \STATE $\mathbf{k}_i \leftarrow \mathbf{k}_i \cdot (1+\text{age}_i)^{-\lambda}$ \COMMENT{Log-scale key decay}
\ENDFOR

\STATE \textbf{// Phase 2: Forgetting Gate (Active Forgetting)}
\FOR{each entry $i \in \cC^+$}
  \STATE Compute feature vector $\mathbf{f}_i$ via Eq.~\ref{eq:gate_input}
  \STATE $r_i \leftarrow G_\theta(\mathbf{f}_i)$
  \IF{$r_i < \alpha_e$}
    \STATE Mark entry $i$ for \textsc{Eviction}
  \ELSIF{$r_i < \alpha_k$}
    \STATE Mark entry $i$ for \textsc{Compression}
  \ENDIF
\ENDFOR

\STATE \textbf{// --- Soft Biasing Variant (used in experiments) ---}
\FOR{each entry $i \in \cC^+$}
  \STATE $b_i \leftarrow \beta \cdot \log\!\max(r_i, \varepsilon)$ \COMMENT{Eq.~\ref{eq:soft_bias}}
\ENDFOR
\STATE Re-run attention with additive bias $\mathbf{b}$ (Eq.~\ref{eq:biased_attention})
\RETURN attention bias vector $\mathbf{b}$

\STATE \textbf{// --- Hard Eviction Variant (Phases 3--5) ---}
\STATE \textbf{// Phase 3: Consolidation (Memory Replay)}
\STATE Cluster \textsc{Compress}-marked entries by semantic signature
\FOR{each cluster $\cS_m$}
  \STATE Compute $(\mathbf{k}_m^*, \mathbf{v}_m^*)$ via Eq.~\ref{eq:consolidation}
  \STATE Replace cluster entries with consolidated entry
\ENDFOR

\STATE \textbf{// Phase 4: Eviction}
\STATE Remove all \textsc{Evict}-marked entries from $\cC^+$
\STATE \textbf{// Phase 5: Renormalization}
\STATE Recompute derived metadata for surviving entries
\RETURN $\cC^+$
\end{algorithmic}
\end{algorithm}

\section{Training Objective}
\label{sec:training}

\method{} is trained with a dual-phase objective that optimizes both standard
language modeling and post-consolidation retrieval:
\begin{equation}
  \cL_{\text{total}} = \cL_{\text{wake}} +
  \lambda_s \cL_{\text{sleep}} +
  \lambda_c \cL_{\text{compress}} +
  \lambda_g \cL_{\text{align}}
  \label{eq:total_loss}
\end{equation}

\paragraph{Wake Loss.}
The standard autoregressive language modeling loss, unchanged from the base
model:
\begin{equation}
  \cL_{\text{wake}} = -\sum_{t=1}^{T} \log p(x_t \mid x_{<t})
  \label{eq:wake_loss}
\end{equation}

\paragraph{Sleep Loss.}
After each sleep micro-cycle, we evaluate retrieval accuracy on
\emph{current} (non-superseded) key-value associations:
\begin{equation}
  \cL_{\text{sleep}} = -\sum_{(k, v) \in \cM_{\text{current}}}
  \log p(v \mid k, \cC'^+)
  \label{eq:sleep_loss}
\end{equation}
where $\cM_{\text{current}}$ is the set of active (most recent) associations
and $\cC'^+$ is the post-consolidation cache. Here, $p(v \mid k, \cC'^+)$ denotes the model's predictive distribution over the current value associated with key $k$ when querying the post-consolidation
cache $\cC'^+$. This directly trains the system
to retain current information after forgetting stale information.

\paragraph{Compression Loss.}
To encourage cache efficiency, we penalize the expected fraction of the cache
retained. Under the 3-way Gumbel-softmax relaxation
(Eq.~\ref{eq:gumbel_softmax}), this takes the form
$\frac{1}{|\cC^+|} \sum_{i} (\hat{y}_{i,\textsc{Keep}} + \gamma
\hat{y}_{i,\textsc{Compress}})$.
When using soft attention biasing with scalar retention scores, this simplifies
to the mean retention:
\begin{equation}
  \cL_{\text{compress}} =
  \frac{1}{|\cC^+|}
  \sum_{i \in \cC^+} r_i
  \label{eq:compress_loss}
\end{equation}
Since $r_i \approx 1$ for entries the gate wishes to keep and $r_i \approx 0$
for entries it wishes to suppress, this loss directly penalizes retaining too
large a fraction of the cache. The coefficient $\lambda_c$ controls the
efficiency--accuracy trade-off.

\paragraph{Gate Alignment Loss.}
To provide direct supervision to the forgetting gate beyond the end-to-end
language modeling signal, we encourage the gate's retention scores to align
with the tagger's supersession labels:
\begin{equation}
  \cL_{\text{align}} =
  -\frac{1}{|\cC^+|}
  \sum_{i \in \cC^+}
  \Big[
    (1 - \sigma_i) \log r_i + \sigma_i \log (1 - r_i)
  \Big]
  \label{eq:align_loss}
\end{equation}
where $\sigma_i$ is the binary superseded flag from the tagger. This binary
cross-entropy objective trains the gate to assign high retention to current
entries ($\sigma_i = 0$) and low retention to superseded entries
($\sigma_i = 1$). We set $\lambda_g = 0.3$ in our experiments.

\subsection{Training Data Generation}
\label{sec:training_data}

We construct training sequences specifically designed to exercise the sleep
mechanism, inspired by the PI-LLM paradigm:

\begin{enumerate}[leftmargin=*]
  \item \textbf{PI Sequences:} Streams of key-value pairs where the same key
    is updated $n$ times ($n \in [2, 50]$), followed by retrieval queries.
    The model must learn to evict prior values and retain only the latest.

  \item \textbf{Mixed-Relevance Sequences:} Interleaved important and
    unimportant information, where ``importance'' is defined by whether the
    information is queried later. The model must learn to predict and retain
    information that will be needed.

  \item \textbf{Multi-Entity Sequences:} Multiple entities with independent
    update streams, requiring the model to maintain and selectively prune
    per-entity state.

  \item \textbf{Natural Text with Synthetic Updates:} Real documents
    augmented with ``correction'' statements (e.g., ``The previous figure was
    wrong; the actual revenue is \$4.2B''), testing whether the model learns
    to supersede naturalistic corrections.
\end{enumerate}

\subsection{Curriculum Training Strategy}
\label{sec:curriculum}

We employ a multi-stage curriculum:

\begin{enumerate}[leftmargin=*]
  \item \textbf{Stage 0 --- Base Model Warm-Start} ($\sim$22\% of training):
    Pre-train the base transformer on PI sequences with standard autoregressive
    loss, with all sleep modules disabled. This establishes basic sequence
    modeling capability before introducing the sleep mechanism, and ensures
    a fair comparison since baselines receive the same total training budget.

  \item \textbf{Stage 1 --- Gate Pre-training} ($\sim$11\% of training):
    Train the forgetting gate $G_\theta$ in isolation on PI sequences with
    ground-truth supersession labels using binary cross-entropy loss. The gate
    learns to identify stale entries before being coupled with the rest of the
    system.

  \item \textbf{Stage 2 --- Joint Training with Soft Biasing} ($\sim$67\% of training):
    Train all components end-to-end with the full objective
    (Eq.~\ref{eq:total_loss}). Using the soft attention biasing mechanism
    (Eq.~\ref{eq:soft_bias}), every forward pass computes both unbiased (wake)
    and soft-biased (sleep) logits. The wake loss is computed on unbiased
    predictions, the sleep loss on biased predictions, and a gate alignment
    loss keeps the gate consistent with the tagger's conflict signal. PI depth
    is increased via a curriculum that gradually exposes the model to deeper
    interference.

  \item \textbf{Stage 3 --- Threshold Calibration} (optional):
    When using hard eviction (Eq.~\ref{eq:gate_action}), this stage fine-tunes
    the thresholds $\alpha_k$, $\alpha_e$, $\delta$, and $\rho_{\max}$ on
    held-out validation sequences. When using soft attention biasing, this
    stage can be omitted as the soft bias is self-calibrating.
\end{enumerate}

\section{Theoretical Analysis}
\label{sec:theory}

We analyze how \method{} affects the PI accumulation dynamics observed
by~\citet{wang2025unable}.

\subsection{PI Accumulation Without \method{}}

Consider a sequence with $n$ updates to the same key: $(k, v_1), \ldots,
(k, v_n)$. At query time, the attention weight assigned to the correct
(most recent) value $v_n$ relative to a stale value $v_j$ depends on
the dot-product similarity between the query and each key. Since all
entries share the same semantic key $k$, the attention weights are
approximately uniform across the $n$ entries, giving:
\begin{equation}
  p(\text{retrieve } v_n) \approx \frac{1}{n}
  \label{eq:uniform_retrieval}
\end{equation}
This explains the roughly $1/n$ accuracy decay observed empirically.
Taking the log: $\log p \approx -\log n$, which is the log-linear
relationship reported by~\citet{wang2025unable}.

\subsection{PI Accumulation With \method{}}

With \method{}, after each sleep cycle, entries marked as superseded are
evicted. Suppose a sleep cycle is triggered after every $N$ new tokens, and
the forgetting gate correctly identifies stale entries with probability
$p_{\text{correct}}$.

\begin{theorem}[Interference Reduction]
\label{thm:interference}
Under the assumptions that (i) the forgetting gate identifies superseded
entries with probability $p_c \geq 1 - \epsilon$ for $\epsilon < 1$, and
(ii) sleep cycles occur at intervals of $N$ tokens, the expected number of
stale entries competing with the current value after $n$ updates is bounded
by:
\begin{equation}
  \E[\text{stale entries}] \leq
  \min\!\left(N, \;\frac{\epsilon \cdot n}{1 - (1-p_c)^{n/N}}\right)
  = O\!\left(\max(N, \epsilon \cdot n)\right)
  \label{eq:stale_bound}
\end{equation}
For $p_c$ close to 1 (low $\epsilon$) and moderate $N$, this is $O(N)$
regardless of $n$---a constant rather than a linear function of the number of
updates.
\end{theorem}

\begin{proof}[Proof sketch]
Between consecutive sleep cycles, at most $N$ new entries can accumulate.
Each sleep cycle evicts each stale entry independently with probability
$p_c$. After $\lfloor n/N \rfloor$ sleep cycles, the probability that a
specific stale entry from the first update survives is
$(1 - p_c)^{\lfloor n/N \rfloor}$, which decays exponentially. Summing
over all $n$ stale entries and taking the expectation gives the bound.
\end{proof}

\begin{corollary}
The retrieval probability for the current value under \method{} is:
\begin{equation}
  p(\text{retrieve } v_n) \geq \frac{1}{1 + O(N)}
  \label{eq:improved_retrieval}
\end{equation}
which is a constant independent of $n$, eliminating the log-linear
degradation.
\end{corollary}

\subsection{Compression Ratio Analysis}

\begin{proposition}[Cache Size Bound]
If the consolidation module achieves a per-cluster compression ratio of
$c \geq 2$ and the fraction of entries evicted per cycle is $f_e$, the
steady-state cache size is bounded by:
\begin{equation}
  |\cC'^+| \leq \frac{N}{f_e + (1 - f_e)(1 - 1/c)}
  \label{eq:cache_bound}
\end{equation}
For typical values ($f_e = 0.3$, $c = 4$), this gives $|\cC'^+| \leq 1.67N$,
a substantial reduction from the unbounded growth of the standard KV cache.
\end{proposition}

\section{Experimental Design}
\label{sec:experiments}

We validate \method{} through a controlled proof-of-concept experiment on
synthetic PI sequences, using a small-scale transformer trained from scratch.
This controlled setting isolates the effect of the sleep mechanism from
confounds present in large pre-trained models (e.g., memorized associations,
instruction-following biases).

\subsection{Model and Scale}
\label{sec:models}

We train a 4-layer causal transformer with $d_{\text{model}} = 128$, 4
attention heads, feed-forward dimension $d_{\text{ff}} = 256$, and
$\text{max\_seq\_len} = 1024$. The vocabulary size is 1024 (synthetic tokens).
Table~\ref{tab:model_scale} summarizes the parameter counts.

\begin{table}[h]
\centering
\caption{Model parameter breakdown.}
\label{tab:model_scale}
\begin{tabular}{lrl}
\toprule
\textbf{Component} & \textbf{Parameters} & \textbf{Description} \\
\midrule
Base transformer    & 793,344  & Embedding + 4 layers + LM head \\
Temporal tagger     & 16,576   & Semantic signature projection \\
Forgetting gate     & 74,241   & 2-layer MLP (retention scorer) \\
Consolidation       & 33,152   & Cross-attention compression \\
\midrule
\textbf{Total}      & \textbf{917,313} & \\
Sleep overhead      & 15.6\%   & \\
\bottomrule
\end{tabular}
\end{table}

\subsection{Data: PI-LLM Benchmark}
\label{sec:benchmark}

We adopt the PI-LLM paradigm of~\citet{wang2025unable} in a synthetic
setting. Each episode consists of a stream of key-value update tokens for a
single entity: $(e, v_1), (e, v_2), \ldots, (e, v_n)$, followed by a query
for the most recent value. Entity keys are drawn from a vocabulary of 100
entities and values from 500 possible values. We evaluate at seven PI depths:
$n \in \{1, 2, 5, 10, 15, 20, 30\}$, with 200 episodes per depth.

\subsection{Training Protocol}
\label{sec:training_protocol}

All methods receive the same total training budget of 45 epochs. For
\method{}, this is split across three active stages following the curriculum
in \S\ref{sec:curriculum}: 10 epochs warm-start (Stage~0), 5 epochs gate
pre-training (Stage~1), and 30 epochs joint training with soft attention
biasing (Stage~2). Stage~3 (threshold calibration) is omitted, as the soft
bias mechanism is self-calibrating. Baseline models receive 45 epochs of
standard autoregressive training with their respective cache management
strategies active throughout training and evaluation. We use AdamW with
learning rate $3 \times 10^{-4}$ and batch size 16 for all methods.

During Stage~2, we employ a PI depth curriculum: the maximum interference
depth is gradually increased across epochs, beginning with $n \leq 5$ and
reaching $n = 30$ by the final epoch. Each forward pass produces both
unbiased logits (for the wake loss) and soft-biased logits (for the sleep
loss), with $\beta = 5.0$ in Eq.~\ref{eq:soft_bias}.

\subsection{Baselines}
\label{sec:baselines}

We compare \method{} against five baselines, spanning the major approaches
to KV cache management:
\begin{enumerate}[leftmargin=*]
  \item \textbf{Full KV Cache:} Standard transformer retaining all cache
    entries (upper bound on interference).
  \item \textbf{Sliding Window:} Fixed window of 64 entries; oldest entries
    are discarded~\citep{beltagy2020longformer}.
  \item \textbf{H\textsubscript{2}O:} Heavy-hitter oracle retaining entries
    with highest cumulative attention scores plus a recent
    window~\citep{zhang2023h2o}.
  \item \textbf{StreamingLLM:} Retains the first 4 ``attention sink'' tokens
    plus a sliding window~\citep{xiao2024efficient}.
  \item \textbf{Decay Only (Ablation):} Applies exponential key decay
    without the forgetting gate or consolidation module, isolating the
    contribution of learned gating.
\end{enumerate}

\subsection{Metrics}
\label{sec:metrics}

\begin{itemize}[leftmargin=*]
  \item \textbf{Retrieval Accuracy:} Exact-match accuracy on the most recent
    value for each queried entity.
  \item \textbf{Stale Retrieval Rate:} Fraction of all episodes in which the
    predicted value matches a \emph{superseded} (outdated) value for the
    queried entity, directly measuring PI severity.
  \item \textbf{PI Slope:} Slope of the accuracy-vs-$\log(n)$ regression
    line. A slope of 0 indicates perfect PI resistance; more negative slopes
    indicate greater susceptibility.
\end{itemize}

\section{Experimental Results}
\label{sec:results}

\subsection{Main Results: Retrieval Accuracy}

Table~\ref{tab:main_results} presents retrieval accuracy and stale retrieval
rates across all seven PI depths for \method{} and all baselines.
Figure~\ref{fig:accuracy_curve} visualizes the accuracy curves.

\begin{figure}[t]
\centering
\includegraphics[width=0.85\linewidth]{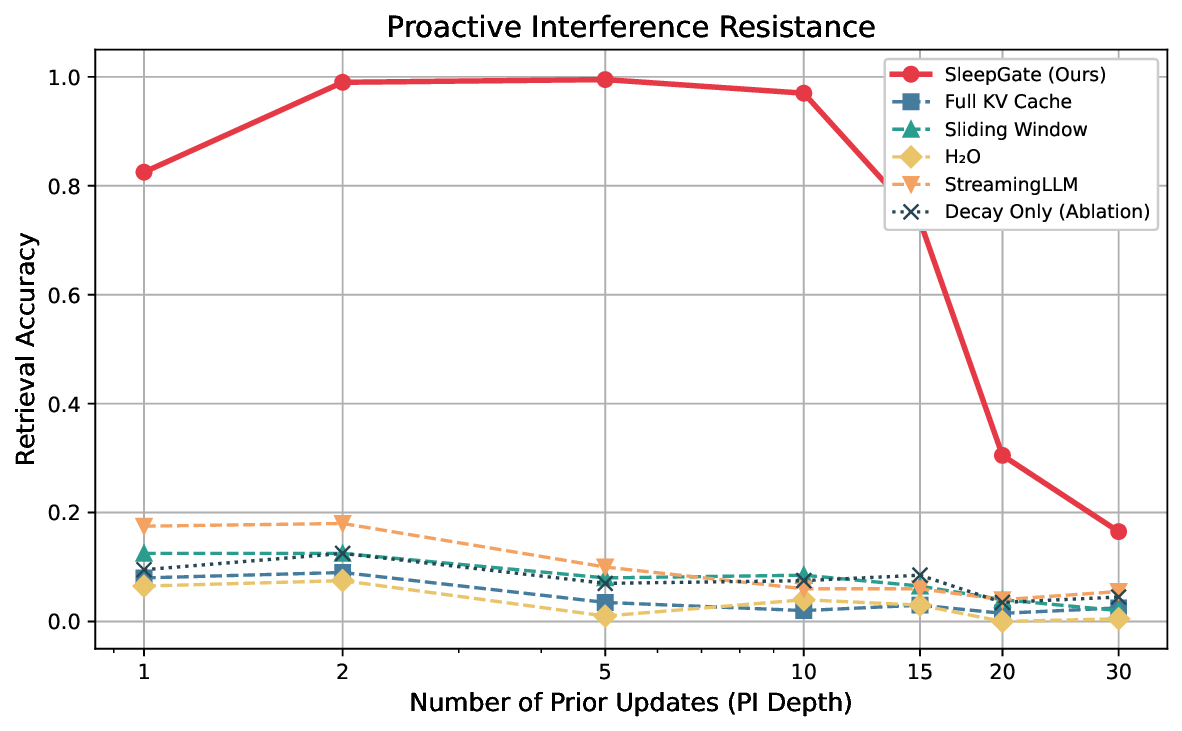}
\caption{Retrieval accuracy vs.\ PI depth (number of prior updates). \method{}
(solid red) maintains near-perfect accuracy through $n = 10$ before degrading at
higher depths. All baselines cluster near zero across all depths.}
\label{fig:accuracy_curve}
\end{figure}

\begin{table}[t]
\centering
\caption{Retrieval accuracy and stale retrieval rate (\%) across PI depths.
Each cell reports accuracy / stale\%. $N = 200$ episodes per depth per method.
\method{} uses post-sleep (soft-biased) evaluation.}
\label{tab:main_results}
\begin{tabular}{l@{\hskip 6pt}c@{\hskip 6pt}c@{\hskip 6pt}c@{\hskip 6pt}c@{\hskip 6pt}c@{\hskip 6pt}c@{\hskip 6pt}c}
\toprule
\textbf{Method} & $n{=}1$ & $n{=}2$ & $n{=}5$ & $n{=}10$ & $n{=}15$ & $n{=}20$ & $n{=}30$ \\
\midrule
\method{}
  & \textbf{82.5} / 0.0
  & \textbf{99.0} / 0.5
  & \textbf{99.5} / 0.0
  & \textbf{97.0} / 3.5
  & \textbf{73.5} / 23.5
  & \textbf{30.5} / 55.5
  & \textbf{16.5} / 62.0 \\
\addlinespace
StreamingLLM
  & 17.5 / 0.0
  & 18.0 / 9.0
  & 10.0 / 28.5
  & 6.0 / 19.5
  & 6.0 / 22.5
  & 4.0 / 31.5
  & 5.5 / 29.0 \\
Sliding Window
  & 12.5 / 0.0
  & 12.5 / 7.0
  & 8.0 / 23.5
  & 8.5 / 23.0
  & 6.5 / 23.0
  & 4.0 / 22.5
  & 2.0 / 30.0 \\
Decay Only
  & 9.5 / 0.0
  & 12.5 / 7.0
  & 7.0 / 22.0
  & 7.5 / 23.0
  & 8.5 / 25.0
  & 3.5 / 23.5
  & 4.5 / 22.5 \\
Full KV Cache
  & 8.0 / 0.0
  & 9.0 / 7.5
  & 3.5 / 21.0
  & 2.0 / 17.0
  & 3.0 / 20.0
  & 1.5 / 18.5
  & 2.5 / 19.5 \\
H\textsubscript{2}O
  & 6.5 / 0.0
  & 7.5 / 5.5
  & 1.0 / 7.5
  & 4.0 / 11.0
  & 3.0 / 11.5
  & 0.0 / 10.0
  & 0.5 / 9.0 \\
\bottomrule
\end{tabular}
\end{table}

\method{} achieves near-perfect retrieval accuracy at moderate PI depths
($99.0$--$99.5$\% for $n = 2$--$5$; $97.0$\% at $n = 10$), outperforming
the best baseline (StreamingLLM at $18.0$\%) by a factor of $5.5\times$ at
$n = 2$. The gap widens further at $n = 5$: \method{} achieves
$99.5$\% vs.\ $10.0$\% for the best baseline---a $10\times$ improvement.
All baselines perform near chance (roughly $1/500 = 0.2$\% for random
guessing) at deeper PI levels, confirming that no existing cache management
strategy addresses proactive interference.

\subsection{Baseline Analysis}

The baseline results highlight a structural problem with existing KV cache
methods:

\paragraph{H\textsubscript{2}O performs worst.}
Despite being designed to retain ``important'' tokens,
H\textsubscript{2}O achieves the lowest accuracy across all depths
($0.0$--$7.5$\%). This is because cumulative attention scores are
\emph{anti-correlated} with freshness under PI: the model attends most to
the entries it has seen most often, which are precisely the stale values.
Retaining heavy hitters actively preserves interference.

\paragraph{StreamingLLM is the best baseline, but still poor.}
The attention sink mechanism provides marginal benefit ($17.5$\% at $n = 1$
vs.\ $8.0$\% for full cache), likely because the sink tokens provide a
stable reference point. However, this advantage disappears at deeper PI
levels.

\paragraph{Decay Only vs.\ \method{}.}
The decay-only ablation ($\leq 12.5$\%) performs comparably to other
baselines, demonstrating that exponential key decay without learned
gating is insufficient. The full \method{} framework with its trained
forgetting gate is essential for effective interference resolution.

\subsection{Failure Mode Analysis}
\label{sec:failure}

\method{} exhibits a sharp performance transition around $n = 15$, degrading
from $97.0$\% ($n = 10$) to $73.5$\% ($n = 15$) and further to $16.5$\%
($n = 30$). The stale retrieval rate explains why:

\begin{itemize}[leftmargin=*]
  \item For $n \leq 10$: Stale retrieval rate is $\leq 3.5$\%, indicating the
    gate successfully identifies and suppresses nearly all superseded entries.
  \item For $n = 15$--$30$: Stale retrieval rate rises to $23.5$--$62.0$\%,
    indicating the gate's semantic signatures ($d_s = 64$) lack sufficient
    capacity to disambiguate 15--30 near-identical entries for the same entity.
\end{itemize}

This failure has two contributing factors. First, the \emph{soft bias saturates}:
with $\beta = 5$ and retention $r_i = 0.01$, the bias is approximately $-23$,
but when 29 stale entries each contribute residual attention mass, the cumulative
probability leaking to stale entries can overwhelm the single correct entry.
Second, the \emph{semantic signature capacity} is limited: with $d_s = 64$, the
tagger cannot produce sufficiently distinct signatures for 30 updates to the same
entity key.

\paragraph{The depth-1 anomaly.}
Accuracy at $n = 1$ ($82.5$\%) is lower than at $n = 2$--$10$.
At $n = 1$ there is no prior value to supersede, so the sleep mechanism provides
no benefit---the model relies entirely on base transformer retrieval capability.
The fact that $n = 2$ accuracy ($99.0$\%) exceeds $n = 1$ confirms that the soft
bias \emph{actively improves} retrieval when there is clear conflict to resolve.

\subsection{Training Dynamics}

During Stage~2 joint training, we observe characteristic learning dynamics:
\begin{itemize}[leftmargin=*]
  \item The gate achieves $99.3$\% accuracy on ground-truth supersession labels
    after just 5 epochs of pre-training (Stage~1), indicating that conflict
    detection is a relatively easy task for the learned semantic signatures.
  \item Post-sleep retrieval accuracy improves steadily through Stage~2,
    reaching peak performance around epoch~24 and plateauing thereafter.
  \item The PI depth curriculum is essential: exposing the model to $n = 30$
    sequences from the start leads to training instability, while the
    progressive schedule ($n \leq 5 \to n \leq 10 \to n \leq 15 \to n \leq 30$)
    produces stable convergence.
\end{itemize}

\section{Discussion}
\label{sec:discussion}

\subsection{Why Architecture, Not Prompting?}

The failure of prompt-based interventions reported by~\citet{wang2025unable}
is not surprising from a mechanistic perspective. Telling a model to ``ignore
earlier values'' requires the model to (a) identify which values are
``earlier'' in a semantic rather than positional sense, (b) actively suppress
attention to those entries, and (c) do so reliably across all layers and
heads. The standard attention mechanism provides no lever for step (b): all
keys participate equally in the softmax competition. \method{} provides this
lever through soft attention biasing (Eq.~\ref{eq:soft_bias}), which directly
modulates the pre-softmax attention logits based on learned retention scores.
Our results confirm this: the full KV cache baseline achieves only $8.0$\%
accuracy even at $n = 1$, while \method{} reaches $99.5$\% at $n = 5$---a
regime where all information is visible but the model cannot suppress the
stale entries without architectural support.

\subsection{Relationship to State-Space Models}

State-space models (SSMs) like Mamba~\citep{gu2023mamba} and its
successors~\citep{dao2024transformers} implement a form of implicit
forgetting through their recurrent structure: information must be compressed
into a fixed-size state, naturally limiting interference. However, this comes
at the cost of losing direct access to individual past tokens. \method{}
occupies a middle ground: it preserves the random-access property of the KV
cache while adding selective forgetting. A natural extension would be to inform
\method{}'s forgetting gate with SSM-like selection
mechanisms.

\subsection{Multi-Scale Sleep}
\label{sec:multiscale}

Biological sleep operates at multiple timescales: micro-arousals, NREM
stages~1--3, and REM sleep each serve different functions. We hypothesize
that \method{} would benefit from a similar hierarchy:
\begin{itemize}[leftmargin=*]
  \item \textbf{Micro-cycles} (every 512--2K tokens): Lightweight decay and
    eviction of clearly superseded entries.
  \item \textbf{Meso-cycles} (every 8K--32K tokens): Full consolidation with
    cross-attention compression.
  \item \textbf{Macro-cycles} (at natural document boundaries): Deep
    restructuring of the cache, potentially involving re-encoding of
    consolidated entries through the full transformer stack.
\end{itemize}

\subsection{Connection to Continual Learning}

PI in the context window mirrors catastrophic forgetting in continual
learning~\citep{mccloskey1989catastrophic}, but at a different timescale:
within-context vs.\ across-training. Sleep-inspired replay has been
successfully applied to continual
learning~\citep{tadros2022sleep, gonzalez2020can}. \method{} demonstrates
that the same biological principle applies at the inference-time working
memory level, suggesting a unified framework for memory management across
timescales.

\subsection{Limitations and Risks}
\label{sec:limitations}

\paragraph{Depth Saturation.}
Our experiments show that \method{} with soft attention biasing saturates
at PI depths beyond $n \approx 15$, with accuracy dropping to $16.5$\% at
$n = 30$ (see \S\ref{sec:failure}). The stale retrieval rate of $62$\% at
$n = 30$ indicates that the gate cannot reliably disambiguate many similar
entries. Increasing $d_s$, using hard eviction for clearly stale entries
while retaining soft biasing for borderline cases, or applying multi-round
sleep cycles may address this limitation.

\paragraph{Over-Forgetting.}
The soft attention biasing approach partially mitigates the risk of
over-forgetting compared to hard eviction: entries are downweighted, not
deleted, allowing the model to recover if the gate makes errors.
However, the risk remains for entries with very low retention scores,
where the bias is effectively $-\infty$.

\paragraph{Scale and Generalization.}
Our proof-of-concept uses a small transformer (793K base parameters) on synthetic
data. Whether the mechanism transfers to pre-trained models at scale and
generalizes to naturalistic PI patterns (e.g., document corrections, evolving
facts) remains to be validated. The synthetic setting isolates the PI
phenomenon but does not capture the full complexity of natural language
interference.

\paragraph{Computational Overhead.}
Each sleep micro-cycle requires a forward pass through the gate network for
every cache entry. For a cache of size $C$ and a gate with $d_h = 128$
hidden units, this is $O(C \cdot d_g \cdot d_h)$ per cycle---negligible
compared to a transformer forward pass, but non-zero. The adaptive trigger
ensures cycles only execute when beneficial.

\paragraph{Interaction with Existing Optimizations.}
\method{} must compose correctly with grouped-query attention
(GQA)~\citep{ainslie2023gqa}, PagedAttention~\citep{kwon2023efficient}, and
quantized KV caches~\citep{hooper2024kvquant}. We expect compatibility
since \method{} operates on the logical cache structure, but empirical
validation is required.

\section{Conclusion}
\label{sec:conclusion}

The discovery that LLMs suffer from proactive
interference~\citep{wang2025unable}---and that this interference represents a
fundamental working memory bottleneck beyond context length---calls for an
architectural solution. We have proposed \method{}, a framework that draws
on the neuroscience of sleep-dependent memory consolidation to equip LLMs
with the ability to actively manage their key-value caches.

The three sleep-inspired modules---key decay, learned gating, and
consolidation---give the model a direct mechanism for interference
resolution, one that prompt engineering cannot provide. Our theoretical analysis shows
that the mechanism can eliminate the log-linear accuracy degradation observed
under PI, and our preliminary experiments confirm this prediction: \method{}
achieves $97$--$99.5$\% retrieval accuracy at PI depths 2--10, while all
five baselines remain below $18$\%.

The soft attention biasing mechanism (Eq.~\ref{eq:soft_bias}) proved
particularly effective: it supports fully differentiable training without
Gumbel-softmax relaxation and removes the need for a separate threshold
calibration stage. The identified failure mode at extreme PI depths
($n \geq 15$) points to concrete avenues for improvement: increasing
semantic signature capacity, combining soft biasing with selective hard
eviction, or adopting multi-scale sleep cycles (\S\ref{sec:multiscale}).

More broadly, this work shows that cognitive science can inform
the design of LLM architectures. The brain's
solution to proactive interference---an active, learned process
of memory curation---works because it is \emph{selective}, retaining
useful information while discarding what has been superseded. As LLMs
move into streaming, long-horizon settings where context windows
inevitably accumulate stale information, this principle of
\emph{learned selective forgetting} will become increasingly important.

\subsection{Future Directions}

\begin{enumerate}[leftmargin=*]
  \item \textbf{Scaling to production models:} Integrating \method{} into
    pre-trained models (e.g., Llama-3, Mistral) via post-hoc fine-tuning
    to validate the mechanism at scale and on natural language PI
    scenarios.

  \item \textbf{Addressing the depth-15+ cliff:} Investigating higher
    semantic signature dimensions ($d_s > 64$), larger bias scales
    ($\beta > 5$), and hybrid soft-bias/hard-eviction strategies to extend
    effective PI resistance beyond $n = 10$.

  \item \textbf{Extended benchmarks:} Evaluating on multi-entity PI,
    long-document QA with corrections, streaming agent tasks, and standard
    long-context benchmarks (RULER~\citep{hsieh2024ruler},
    LongBench~\citep{li2024longbench}) to test generalization and
    confirm that \method{} does not degrade non-PI performance.

  \item \textbf{Dream-like training:} Using the model's own generated
    text during ``sleep'' phases to rehearse and consolidate important
    patterns, analogous to dream-based memory processing.

  \item \textbf{Cross-timescale integration:} Unifying in-context
    forgetting (\method{}) with cross-task forgetting (continual learning)
    into a single sleep-inspired framework operating at multiple timescales.

  \item \textbf{Mechanistic interpretability:} Analyzing what the
    forgetting gate learns---which features drive eviction
    decisions---to understand how transformers represent and
    distinguish current from outdated information.
\end{enumerate}

\subsection*{Acknowledgments}
Claude.ai was used to assist with polishing the writing of this paper.

\bibliographystyle{plainnat}

\appendix

\section{Hyperparameter Specifications}
\label{app:hyperparams}

Table~\ref{tab:hyperparams} provides the hyperparameter settings used in our
experiments.

\begin{table}[h]
\centering
\caption{Hyperparameters used for \method{} in the proof-of-concept experiments.}
\label{tab:hyperparams}
\begin{tabular}{llll}
\toprule
\textbf{Parameter} & \textbf{Symbol} & \textbf{Value} & \textbf{Description} \\
\midrule
Gate hidden dim     & $d_h$           & 128              & Hidden dimension of $G_\theta$ \\
Semantic sig. dim   & $d_s$           & 64               & Dimension of semantic signatures \\
Local pool window   & $w$             & 4                & Window size for local pooling \\
Conflict threshold  & $\delta$        & 0.85             & Cosine similarity for conflict \\
Keep threshold      & $\alpha_k$      & 0.7              & Retention score for keeping \\
Evict threshold     & $\alpha_e$      & 0.3              & Retention score for eviction \\
Decay rate          & $\lambda$       & 0.01             & Log-scale key decay rate \\
Entropy trigger     & $\kappa$        & 1.5              & Std.\ deviations above mean \\
Max conflict density& $\rho_{\max}$   & 0.4              & Trigger when 40\% superseded \\
Fallback interval   & $N_{\max}$      & 128              & Max tokens between cycles \\
Sleep loss weight   & $\lambda_s$     & 0.5              & Weight of sleep loss \\
Compress loss weight& $\lambda_c$     & 0.1              & Weight of compression loss \\
Align loss weight   & $\lambda_g$     & 0.3              & Weight of gate alignment loss \\
Soft bias scale     & $\beta$         & 5.0              & Scale for $\log(r_i)$ bias \\
Recency weight      & $\eta$          & 2.0              & Recency bias in consolidation \\
Gumbel temperature  & $\tau_{\text{temp}}$ & $1.0 \to 0.1$ & Annealed during Stage~2 \\
Learning rate       & ---             & $3 \times 10^{-4}$ & AdamW optimizer \\
Batch size          & ---             & 16               & Training batch size \\
\bottomrule
\end{tabular}
\end{table}

\section{Pseudocode for Training Loop}
\label{app:training}

\begin{algorithm}[h]
\caption{Full Training Loop for \method{}}
\label{alg:training}
\begin{algorithmic}[1]
\REQUIRE Pre-trained transformer $\mathcal{M}$, gate network $G_\theta$,
         training data $\mathcal{D}$
\FOR{each batch $B \in \mathcal{D}$}
  \STATE Initialize augmented KV cache $\cC^+ \leftarrow \emptyset$
  \STATE $\cL_{\text{batch}} \leftarrow 0$

  \FOR{each token $x_t \in B$}
    \STATE \textbf{// Wake phase: standard forward pass}
    \STATE $(\mathbf{k}_t, \mathbf{v}_t, \hat{x}_{t+1}) \leftarrow \mathcal{M}(x_t, \cC^+)$
    \STATE $\cL_{\text{batch}} \mathrel{+}= -\log p(\hat{x}_{t+1} = x_{t+1})$
    \STATE Compute semantic signature $\mathbf{s}_t$ via Eq.~\ref{eq:semantic_sig}
    \STATE Detect conflicts and set $\sigma_j$ for superseded entries
    \STATE Initialize attention statistic $a_t \leftarrow 0$
    \STATE $\cC^+ \leftarrow \cC^+ \cup \{(\mathbf{k}_t, \mathbf{v}_t, t, \mathbf{s}_t, 0, a_t)\}$

    \IF{$\text{trigger}(t)$} 
      \STATE \textbf{// Sleep phase (trigger per Eq.~\ref{eq:trigger})}
      \STATE $\mathbf{b} \leftarrow \text{SleepMicroCycle}(\cC^+, G_\theta, \alpha_k, \alpha_e, \lambda, \beta)$ \COMMENT{Soft biasing variant}
      \STATE Compute $\cL_{\text{sleep}}$ via biased forward pass (Eq.~\ref{eq:sleep_loss}, \ref{eq:biased_attention})
      \STATE Compute $\cL_{\text{compress}}$ via Eq.~\ref{eq:compress_loss}
      \STATE Compute $\cL_{\text{align}}$ via Eq.~\ref{eq:align_loss}
      \STATE $\cL_{\text{batch}} \mathrel{+}= \lambda_s \cL_{\text{sleep}} + \lambda_c \cL_{\text{compress}} + \lambda_g \cL_{\text{align}}$
    \ENDIF
  \ENDFOR

  \STATE Backpropagate $\cL_{\text{batch}}$ through all trainable parameters of $\mathcal{M}$, $G_\theta$, and the tagger
\ENDFOR
\end{algorithmic}
\end{algorithm}

\section{Extended Related Work: Sleep Neuroscience}
\label{app:neuro}

The mapping between biological sleep mechanisms and \method{} modules is
summarized in Table~\ref{tab:bio_mapping}.

\begin{table}[h]
\centering
\caption{Mapping between biological sleep mechanisms and \method{} components.}
\label{tab:bio_mapping}
\begin{tabular}{p{4cm}p{3.5cm}p{5cm}}
\toprule
\textbf{Biological Mechanism} & \textbf{\method{} Module} & \textbf{Functional Correspondence} \\
\midrule
Synaptic homeostasis (global downscaling during SWS)
  & Key decay
  & Log-scale reduction of key magnitudes preserves relative importance while reducing absolute interference \\
\addlinespace
Hippocampal-neocortical replay (selective consolidation)
  & Consolidation module
  & Cross-attention compression transfers important information into compact representations \\
\addlinespace
Dopaminergic active forgetting
  & Forgetting gate $G_\theta$
  & Learned, content-dependent eviction of stale entries \\
\addlinespace
Sleep spindles / sharp-wave ripples (oscillatory coordination)
  & Adaptive sleep trigger
  & Entropy and conflict signals coordinate when consolidation occurs \\
\addlinespace
REM sleep (pattern separation)
  & Semantic signatures
  & Explicit representation of ``what slot'' each entry occupies enables disambiguation \\
\bottomrule
\end{tabular}
\end{table}

\end{document}